\documentclass{article}

\usepackage{microtype}
\usepackage{graphicx}
\usepackage{subfigure}
\usepackage{booktabs} 

\usepackage{hyperref}

\usepackage[accepted]{icml2023}

\usepackage{amsmath}
\usepackage{amssymb}
\usepackage{mathtools}
\usepackage{amsthm}

\usepackage[capitalize,noabbrev]{cleveref}

\theoremstyle{plain}
\newtheorem{theorem}{Theorem}[section]

\theoremstyle{definition}
\newtheorem{definition}[theorem]{Definition}

\theoremstyle{remark}
\newtheorem{remark}[theorem]{Remark}

\usepackage[textsize=tiny]{todonotes}

\usepackage{makecell}
\usepackage{multirow}

\usepackage{float}

\usepackage{commath}
\usepackage{nicefrac}       

\theoremstyle{plain}
\newtheorem*{thm32}{Theorem 3.2}

\theoremstyle{definition}
\newtheorem*{def31}{Definition 3.1}

\icmltitlerunning{Adaptive Certified Training: Towards Better Accuracy-Robustness Tradeoffs}

\begin{document}

\twocolumn[
\icmltitle{Adaptive Certified Training: \\Towards Better Accuracy-Robustness Tradeoffs}



\icmlsetsymbol{equal}{*}

\begin{icmlauthorlist}
\icmlauthor{Zhakshylyk Nurlanov}{yyy,comp}
\icmlauthor{Frank R. Schmidt}{comp}
\icmlauthor{Florian Bernard}{yyy}
\end{icmlauthorlist}

\icmlaffiliation{yyy}{University of Bonn}
\icmlaffiliation{comp}{Bosch Center for Artificial Intelligence}

\icmlcorrespondingauthor{Zhakshylyk Nurlanov}{zh.nurlanov@uni-bonn.de}

\icmlkeywords{Certified training, neural network verification}

\vskip 0.3in
]



\printAffiliationsAndNotice{}  

\begin{abstract}
  As deep learning models continue to advance and are increasingly utilized in real-world systems, the issue of robustness remains a major challenge. Existing certified training methods produce models that achieve high provable robustness guarantees at certain perturbation levels. However, the main problem of such models is a dramatically low standard accuracy, i.e. accuracy on clean unperturbed data, that makes them impractical. In this work, we consider a more realistic perspective of maximizing the robustness of a model at certain levels of (high) standard accuracy. To this end, we propose a novel certified training method based on a key insight that training with adaptive certified radii helps to improve both the accuracy and robustness of the model, advancing state-of-the-art accuracy-robustness tradeoffs. We demonstrate the effectiveness of the proposed method on MNIST, CIFAR-10, and TinyImageNet datasets. Particularly, on CIFAR-10 and TinyImageNet, our method yields models with up to two times higher robustness, measured as an average certified radius of a test set, at the same levels of standard accuracy compared to baseline approaches.
\end{abstract}

\section{Introduction}
Deep learning models are successfully deployed in real-world applications, related to speech and image recognition, natural language processing, drug discovery, recommendation systems, and bioinformatics~\citep{balas2019handbook}. However, the issue of robustness poses safety concerns in mission-critical tasks, such as autonomous driving~\citep{zhang2021driving}, since the discovery of adversarial examples~\citep{szegedy2014intriguing}. Methods for improving empirical robustness, such as adversarial training~\citep{madry2018towards}, are effective against a set of adversarial attacks but give no robustness guarantees.
Formal verification techniques ~\citep{muller2022third} can provide worst-case deterministic performance guarantees. Nonetheless, to improve the robustness certificates, specialized {certified training} methods \citep{wong2018scaling,zhang2019crownibp,shi2021fast,muller2023sabr} are necessary. Existing certified training approaches improve worst-case certified accuracy at the cost of drastically {reduced standard accuracy} on clean inputs. This limits the adoption of certifiably robust models in practical applications.

In this work, we tackle the problem of reduced standard accuracy of certified training methods and propose a novel approach, \textbf{A}daptive \textbf{Cert}ified {T}raining (ACERT). It is based on the idea that the input points may have different levels of vulnerability. Therefore, by adapting certification intensity to each input sample individually, ACERT optimizes robustness on all levels. As a result, it achieves increased standard accuracy and robustness, measured as an average certified radius of a test set, advancing the state-of-the-art. Moreover, we propose to evaluate models based on their performance across accuracy-robustness curves, therefore we introduce a new metric called Accuracy-Robustness Tradeoff (ART) score, which combines standard accuracy and average certified radius.

\paragraph{Main contributions} Our main contributions are summarized in the following:
\begin{enumerate}
    \item An empirical study of accuracy-robustness tradeoffs of certified training methods with a new metric, ART score (short for Accuracy-Robustness Tradeoff), combining standard accuracy and average certified robustness of models;
    \item A novel certified training method called ACERT that efficiently finds theoretically sound certification strength for each input sample and effectively maximizes it;
    \item Experiments demonstrating improved accuracy-robustness curves of ACERT over previous state-of-the-art methods across different datasets, such as MNIST, CIFAR-10, and TinyImageNet.
\end{enumerate}

\section{Background}
In the following, we denote a deep neural network (DNN) as a parameterized function $f_\theta: \mathcal{X} \to \mathcal{Y}$ with parameters $\theta$, an input space $\mathcal{X}$, and an output space $\mathcal{Y}$.  Consider a classification problem, where $\mathcal{X} \subset \mathbb{R}^{d_{\text{in}}}$, $\mathcal{Y}=\mathbb{R}^c$ and $f_\theta(x) \in \mathbb{R}^c$ predict numerical scores (i.e., unnormalized \emph{ logits}) for each of the $c$ classes. Hard-label classification $h(f_\theta(x))$ is performed by assigning an index of the highest score, that is $h(f_\theta(x)) := \arg\max_i f_\theta(x)_i$. 

\paragraph{Adversarial Robustness}
A {hard-label classifier} $h$ is \emph{adversarially robust} on an $\ell_p$-ball $\mathcal{B}_p(x, \varepsilon)$ around input $x$ if, given a ground-truth target class $y$, it always predicts $y$ for all inputs in   $\mathcal{B}_p(x, \varepsilon)$, or formally,
\begin{align}
\label{eq:adv_robustness}
    h(f_\theta(x')) &\equiv y, \\ \forall x' \in \mathcal{B}_p(x, \varepsilon) &:= \{x' \in \mathcal{X}: ||x-x'||_p \leq \varepsilon\}. \notag
\end{align}

Equivalently, a {soft classifier} $f_\theta$ is \emph{adversarially robust} on an $\ell_p$-ball $\mathcal{B}_p(x, \varepsilon)$ around input $x$, if the maximum value of a margin function $m(x', y; f_\theta):=\max_{i \neq y} f_\theta(x')_i - f_\theta(x')_y$ on  $x' \in \mathcal{B}_p(x, \varepsilon)$ is less than~0, i.e.
\begin{equation}
\label{eq:robust_margin}
    R(x, y; \mathcal{B}_p(x, \varepsilon); f_\theta) := \max_{x' \in \mathcal{B}_p(x, \varepsilon)} m(x', y; f_\theta) < 0.
\end{equation}
A solution to the maximization problem in \eqref{eq:robust_margin} defines a \emph{robust margin} function $R(x, y; \mathcal{B}_p(x, \varepsilon); f_\theta)$, which indicates the robustness of the underlying classifier $f_\theta$. 

In this work, we consider a common case of $p=\infty$ for $\ell_p$ balls, i.e. $\mathcal{B}(x, \varepsilon) = \mathcal{B}_\infty(x, \varepsilon)$.  Furthermore, whenever it is clear from context, we omit extra arguments of the robust margin, i.e.~we write $R(x, \varepsilon)~=~R(x, y; \mathcal{B}_p(x, \varepsilon); f_\theta)$.

\paragraph{Certified Robustness}
Calculating the robust margin function in~\eqref{eq:robust_margin} is a challenging task, as it is in general an NP-hard problem~\citep{katz2017reluplex}. Therefore, to guarantee robustness, it is preferable to calculate an upper bound of the robust margin, called the \emph{certified robust margin} denoted as $R_\text{cert} \geq R$. 
The calculation of the certified robust margin may become significantly easier, depending on the tightness of the upper bounds. The certified robust margin is often computed with the help of the outer bounds of the neural network output, namely 
\begin{equation}
\label{eq:certified_robust_margin}
     R_\text{cert}(x, \varepsilon) := \max_{i \neq y} \overline{f_\theta}(x, \varepsilon)_i - \underline{f_\theta}(x, \varepsilon)_y,
\end{equation}
where 
$
    \underline{f_\theta}(x, \varepsilon) \leq f_\theta(x') \leq \overline{f_\theta}(x, \varepsilon),  \forall x' \in  \mathcal{B}(x, \varepsilon)
    $ represent the outer bounds.

With that, the classifier $f_\theta$ is \emph{certifiably robust} on $\mathcal{B}(x, \varepsilon)$ if the upper bound of the robust margin $R_\text{cert}$ is less than 0, i.e.
\begin{equation}
    \label{eq:certified_robustness}
    R(x, \varepsilon) \leq R_\text{cert}(x, \varepsilon) < 0.
\end{equation}

\paragraph{Certified Training}
It can be difficult to successfully verify the robustness of regularly trained neural networks, as they often have low certification rates. To address this issue, specialized training techniques called \emph{certified training}, are used to improve the robustness guarantees of these networks. 

A particularly effective certified training method is based on the interval bound propagation (IBP) technique~\citep{gowal2019ibp}. It propagates the input interval bounds  $\mathcal{B}(x, \varepsilon)$ through the layers of a DNN using interval arithmetic to get the output intervals $(\underline{f_\theta}(x, \varepsilon), \overline{f_\theta}(x, \varepsilon))$ (details of bound propagation in IBP are provided in the Appendix~\ref{appendix:ibp}). Calculating the IBP bounds is comparable to computing a forward pass through an extended network. Therefore, computational efficiency together with favorable optimization properties~\citep{jovanovic2022paradox} allows one to train certifiably robust models by directly optimizing the loss function over the IBP bounds. 

In practice, IBP-based certified training minimizes a differentiable approximation of the certified robust margin function, a cross-entropy loss 
\begin{equation}
    \label{eq:certified_training}
    \mathcal{L}_\text{rob}^\text{IBP}(x, y, \varepsilon_t; \theta) := \mathcal{L}_\text{CE}(\hat{f}_\theta(x, y, \varepsilon_t), y),
\end{equation}
where the radius $\varepsilon_t$ is shared across inputs and is heuristically scheduled to increase during training from 0 to $\varepsilon_{\max}$. The worst-case bounds of a DNN's output $\hat{f}_\theta(x, y, \varepsilon_t)$ in~\eqref{eq:certified_training} are formed such that
\begin{equation}\label{eq:worst_case_bounds}
     \hat{f}_\theta(x, y, \varepsilon_t)_i := \begin{cases}
        \overline{f_\theta}(x, \varepsilon_t)_i, & \forall i\neq y, \\
        \underline{f_\theta}(x, \varepsilon_t)_y, &  i=y.
    \end{cases}
\end{equation}

Optimizing the robust loss outlined in~\eqref{eq:certified_training} yields models that can achieve an impressive level of certified accuracies at certain (large) perturbation levels. However, this often leads to a decrease in accuracy on clean inputs (also known as \emph{standard accuracy}) due to overregularization. Balancing robustness and accuracy remains a challenge for these methods.

\section{Method}
In the following, we propose the \textbf{A}daptive \textbf{Cert}ified training method, ACERT, which addresses the issue of low standard accuracy in existing certified training approaches. First, we define the concept of certified robust radius, which is an integral part of our method. Then, we formalize the adaptive certified training, and finally, present a new metric to capture progress in both accuracy and robustness.

\subsection{Certified Robust Radius}
\begin{definition}[Certified robust radius]\label{def:certified_robust_radius}
    A bounded certified robust radius $\mathcal{E}(x) = \mathcal{E}(x,y; f_\theta; p)~\in~[0, \varepsilon_{\max}]$ of an input sample $(x, y) \sim \mathcal{D}$ with respect to the classifier $f_\theta$ is the maximum radius of an $\ell_p$ ball $\mathcal{B}(x, \varepsilon)$, not exceeding the predefined (perceptual) limit $\varepsilon_{\max}$, where the classifier is certifiably robust. Formally,
    \begin{equation}\label{eq:def_certified_radius}
        \mathcal{E}_\text{cert}(x) = \begin{cases}
            0,& \text{if } R_\text{cert}(x, 0) \geq 0,\\
            \varepsilon_{\max}, & \text{if } R_\text{cert}(x, \varepsilon_{\max}) < 0,\\
            \sup\limits_{R_\text{cert}(x, \varepsilon) < 0} \varepsilon, & \text{otherwise}.
        \end{cases}
    \end{equation}
\end{definition}

\begin{theorem}\label{theorem:certified_radius}
    Let $R_\text{cert}(x, \cdot)$ be continuous and strictly monotonically increasing {function}. Then the certified robust radius, defined in Def.~\ref{def:certified_robust_radius}, exists and is uniquely defined as 
    \begin{equation}\label{eq:certified_radius_theorem}
        \mathcal{E}_\text{cert}(x) = \begin{cases}
            0,& \text{if } R_\text{cert}(x, 0) > 0,\\
            \varepsilon_{\max}, & \text{if } R_\text{cert}(x, \varepsilon_{\max}) < 0,\\
            \varepsilon \text{ s.t. } R_\text{cert}(x, \varepsilon) = 0, & \text{otherwise}.
        \end{cases}
    \end{equation}
\end{theorem}
The existence and uniqueness of the certified robust radius follow from continuity and strict monotonicity of $R_\text{cert}(x, \cdot)$. Find detailed proofs in the Appendix~\ref{appendix:proofs}. 
Further, whenever it is clear, we omit the extra arguments in the certified robust radius function, i.e. $\mathcal{E}_\text{cert}(x)~=~\mathcal{E}_\text{cert}(x,y;f_\theta,p)$.

\begin{remark}\label{remark:ibp_certified_radius}
    We observe that in the regularly initialized classifiers $f_\theta$, interval bound propagation-based certified margin function $R_\text{cert}(x, \cdot)$ in practice satisfies the requirements of Theorem~\ref{theorem:certified_radius}.
\end{remark}
As a result, using IBP for computing the outer bounds of a DNN's output, we can find the certified robust radius of input $x$ as a root of a scalar function $R_\text{cert}(x, \cdot)$, bounded by perceptual limits $[0, \varepsilon_{\max}]$.

\subsection{Adaptive Certified Training}
Employing the definition of certified robust radius from Theorem~\ref{theorem:certified_radius}, we show that it is tractable to directly maximize the radius of each input sample with gradient-based techniques.

\begin{theorem}[Implicit function]\label{theorem:implicit_function}
    Let $R_\text{cert}(\cdot, \cdot)$ be a continuously differentiable function of arguments $(w, \varepsilon)$, and let $(w_0, \varepsilon_0)$ be  a point such that $R_\text{cert}(w_0, \varepsilon_0) = 0$. If $\partial_\varepsilon R_\text{cert}(w_0, \varepsilon_0) \neq 0$,  by implicit function theorem
    \begin{align*}
        \partial_w \mathcal{E}_\text{cert}(w) = - \frac{\partial_w R_\text{cert}(w, \mathcal{E}_\text{cert}(w))}{\partial_\epsilon R_\text{cert}(w, \mathcal{E}_\text{cert}(w))}
    \end{align*}
\end{theorem}
\begin{remark}
    Note that the derivative ${\partial_\varepsilon R_\text{cert}(w, \mathcal{E}_\text{cert}(w))}$ is always a positive scalar for a strictly increasing differentiable function, and that the first argument $w$ can represent both input $x$ and parameters $\theta$. Hence, Theorem~\ref{theorem:implicit_function} implies that a gradient step in a direction of maximizing the certified robust radius $\partial_\theta \mathcal{E}_\text{cert}(\theta)$ is a scaled gradient step in a direction of minimizing the certified margin ${\partial_\theta R_\text{cert}(\theta, \mathcal{E}_\text{cert}(\theta))}$.
\end{remark}

Therefore, by minimizing the certified margin at the point of certified robust radius, i.e. $\min_\theta R_\text{cert}(x, \mathcal{E}_\text{cert}(x); \theta)$,  we equivalently maximize the corresponding certified robust radius $\max_\theta \mathcal{E}_\text{cert}(x; \theta)$. In practice, similar to IBP-based certified training~\eqref{eq:certified_training}, we minimize an approximation of the certified margin that is better suitable for gradient-based optimizers, a cross-entropy loss~$\mathcal{L}_\text{CE}$. 

We employ the worst-case output bounds based on IBP as in Eq.~\eqref{eq:worst_case_bounds}, with the key difference that the input bounds $\mathcal{B}(x, \varepsilon)$ have adaptive radii $\varepsilon$, which are equal to certified robust radii for each of the inputs $ \mathcal{E}_\text{cert}(x)$, that is $\hat{f}_\theta(x, y, \varepsilon) \equiv \hat{f}_\theta(x, y, \mathcal{E}_\text{cert}(x))$. Overall, the robust loss of the Adaptive Certified training method is defined as
\begin{equation}
    \label{eq:ACET_loss}
    \mathcal{L}_\text{rob}^\text{ACERT}(x, y; \theta) := \mathcal{L}_\text{CE}(\hat{f}_\theta(x, y, \mathcal{E}_\text{cert}(x)), y),
\end{equation}

To capture the tradeoff between standard accuracy and robustness, we construct a final loss as a convex combination of standard cross entropy loss and a robust loss, both for IBP and ACERT, i.e.
\begin{align}
    \mathcal{L}_\text{stn}(x, y; \theta) &:= \mathcal{L}_\text{CE}(f_\theta(x), y), \label{eq:standard_loss} \\
    \mathcal{L} &= \kappa \mathcal{L}_\text{stn} + (1 - \kappa) \mathcal{L}_\text{rob}. \label{eq:total_loss}
\end{align}
By varying the proportion of standard loss $\kappa \in [0, 1]$, we can get the accuracy-robustness curves of certified training methods (e.g. see in Fig.~\ref{fig:accuracy_robustness_curves}).

\begin{figure*}[htb]
\centering
\begin{tabular}{cc}
\includegraphics[width=0.43\linewidth]{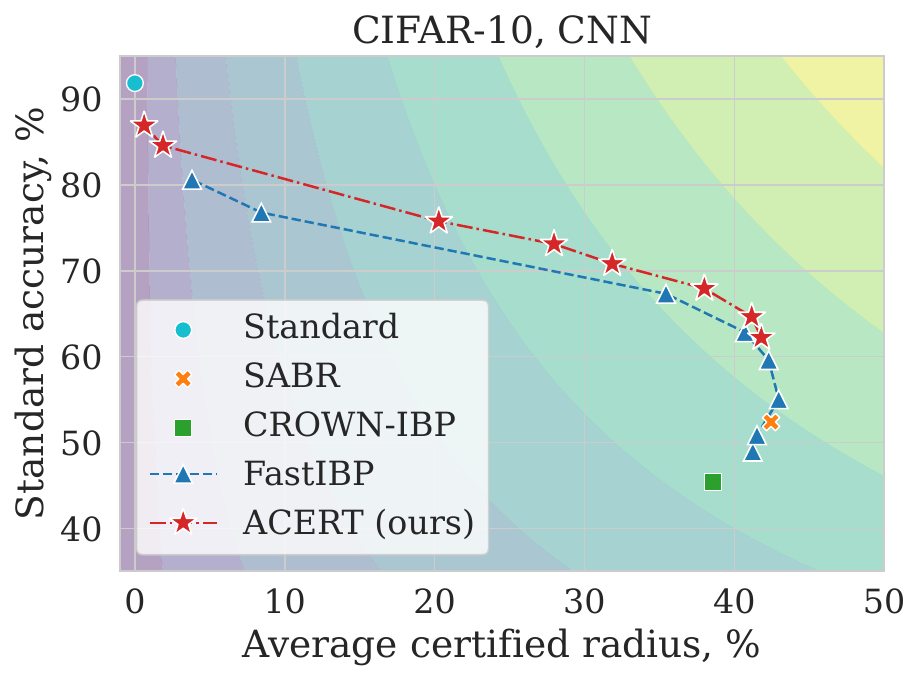} &
\includegraphics[width=0.43\linewidth]{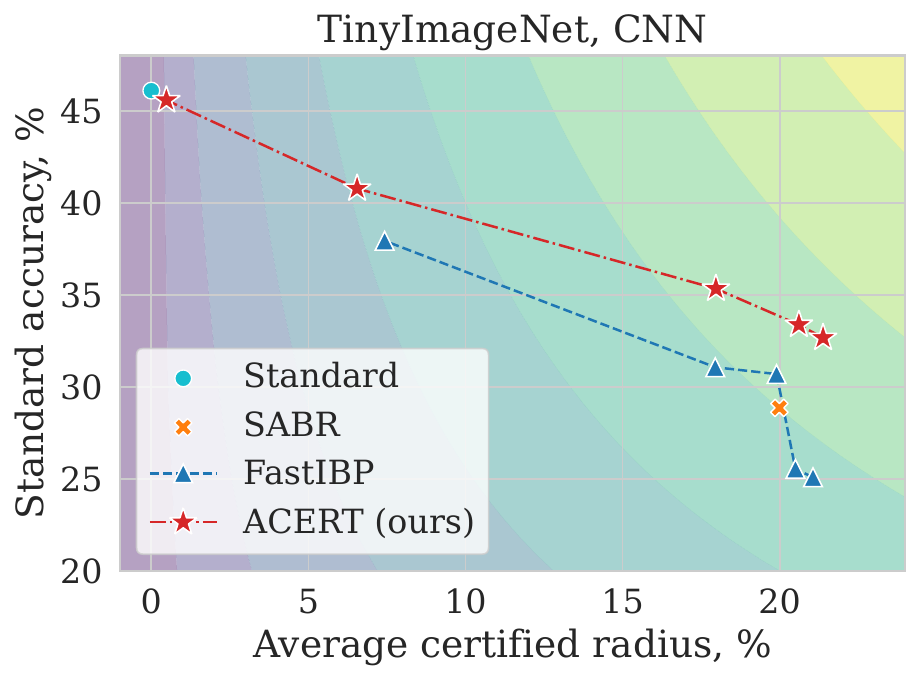} 
\end{tabular}
\caption{Accuracy vs robustness (given as an average certified robust radius of a test set) tradeoffs of certified training methods produced by varying the proportion of standard loss $\kappa \in [0,1]$ in a total loss function~\eqref{eq:total_loss}. The ART score~\eqref{eq:ART}, or geometric mean of accuracy and average certified radius, is visualized with a colormap, where lighter colors represent higher values.}
\label{fig:accuracy_robustness_curves}
\end{figure*}

\subsection{Accuracy-Robustness Tradeoff Score}
In this section, we present metrics that we use for measuring both the accuracy and robustness of models.

The standard accuracy on clean data is the proportion of correctly classified samples (Acc). 

The average robustness of a model is characterized by an average certified robust radius, which is bounded in an interval $[0, \varepsilon_{\max}]$. To compare performance across different perceptual limits~$\varepsilon_{\max}$, we normalize the average certified radius (ACR), that is
\begin{equation}
\label{eq:ACR}
    \text{ACR} := \dfrac{1}{\varepsilon_{\max}} \cdot \dfrac{\sum_i^N \mathcal{E}_\text{cert}(x_i, y_i; f_\theta)}{N}.
\end{equation}
To measure the progress in both standard accuracy and robustness, we propose to use a new metric, ART ({A}ccuracy-{R}obustness {T}radeoff ) that performs a geometric averaging of standard accuracy and average certified radius. Formally,
\begin{equation}
\label{eq:ART}
    \text{ART}(\text{Acc}, \text{ACR}) := \sqrt{\text{Acc} \times \text{ACR}}.
\end{equation}
The main advantage of the geometric mean over other averaging functions is that it is \emph{homogeneous} as a function of each argument. Therefore, changing a scaling factor in the computation of an average certified radius~\eqref{eq:ACR} would only result in a different multiplicative constant of ART, and would not alter the order of compared methods. Formally, 
    $\text{ART}(z_1, c \cdot z_2) = \sqrt{c} \cdot  \text{ART}(z_1, z_2).$
Introducing the ART score helps to select models from a Pareto front of the accuracy-robustness curve.

\section{Experiments}
In the following section, we demonstrate the performance of our proposed method, ACERT, and compare it with state-of-the-art certified training methods on various datasets. Furthermore, we investigate the effect of each component of our method in ablation studies.

\paragraph{Experimental Setup} We conduct experiments on three image classification datasets of various difficulties and scales: MNIST~\citep{LeCun2005TheMD}, CIFAR-10~\citep{Krizhevsky2009LearningML} and TinyImageNet~\citep{le2015tiny}.

The baseline method, FastIBP \citep{shi2021fast}, gives strong IBP-based certified robustness requiring relatively short training schedules. We reimplement it based on the authors' code and run experiments in the same setups as our method and for the same values of $\kappa \in [0, 1)$. For CROWN-IBP~\citep{zhang2019crownibp} and a more recent SABR~\citep{muller2023sabr}, we calculate certified robust radii with IBP on the pre-trained checkpoints provided by authors.

The training procedures, such as IBP initialization, regularizers, and training schedules, follow the prior work, FastIBP~\citep{shi2021fast}. The \texttt{CNN7} architecture is used across datasets, similar to~\citep{shi2021fast,muller2023sabr}. For our method, in the initial stages of training, we schedule the maximally allowed certified radius $\varepsilon_{\max, t}$ to increase from 0~to~$\varepsilon_{\max}$. The verification of all methods is performed by the efficient IBP technique since considered models are trained using the IBP bounds. 
Find more details on hyperparameters, hardware, training times, and reproducibility in the Appendix~\ref{appendix:experimental_details}. We will release the code together with the paper publication.

We implement an efficient vectorized GPU-accelerated version of the   BrentQ~\citep{brent1973algorithmsfor} algorithm to find adaptive certified radii~\eqref{eq:certified_radius_theorem} as roots of a batch of scalar functions $(R_\text{cert}(x^1, \cdot), \ldots, R_\text{cert}(x^b, \cdot))$. The BrentQ algorithm guarantees convergence with any predefined level of error tolerance, and it is generally considered a powerful root finding routine~\citep{press2007numerical}. In practice, it usually requires 2 to 4 iterations to achieve solutions with absolute error tolerance \texttt{xtol=1e-6} and relative tolerance \texttt{rtol=1e-4}. Therefore, this routine can be used during the training of a neural network, allowing one to efficiently compute certified radii with high precision. A more detailed analysis of the root finding algorithm is given in Section~\ref{subsec:root_finder}.

\subsection{Main Results}
\label{subsec:main_results}

We compare ACERT to state-of-the-art certified training methods in terms of standard accuracy, average certified robust radius on a test dataset, and a combination of those, ART score, in Fig.~\ref{fig:accuracy_robustness_curves} and Table~\ref{tab:main_results}.  
Based on the accuracy-robustness curves presented in Fig.~\ref{fig:accuracy_robustness_curves}, it can be inferred that the proposed method, ACERT, demonstrates higher levels of robustness at the same levels of accuracy (Fig.~\ref{fig:accuracy_robustness_curves}) comparing to baseline.
For instance, at a fixed level of standard accuracy of 75\% on CIFAR-10, ACERT provides twice ($\times 2$) of the average certified radius given by FastIBP (18\% vs 9\%). Similarly, ACERT reaches around twice ($\times 2$) of the robustness level of the baseline method on TinyImageNet (14\% vs 7\%) for a standard accuracy of $38\%$. Additionally, we observe in Fig.~\ref{fig:accuracy_robustness_curves} that the accuracy-robustness curves of ACERT consistently get close to the regularly trained network's point with high accuracy, while the baseline fails to reach those levels. This indicates that adaptive training successfully eliminates unnecessary overregularization, which is present in certified training.
Overall, the accuracy-robustness curves of ACERT dominate those of the baseline across datasets, resulting in superior Pareto fronts.

For its part, the baseline method, FastIBP, exhibits accuracy and robustness levels similar to those of SABR, provided that $\kappa$ (the proportion of standard loss) is tuned and the same verification technique is utilized, and it outperforms CROWN-IBP. So analyzing the performance of our method in comparison to FastIBP gives a fair judgment.

\begin{table*}[tb]
\caption{Comparison of standard accuracy ($\text{Acc}$), average certified radius ($\text{ACR}$) and accuracy-robustness tradeoff score ($\text{ART}$) for certified training methods on MNIST, CIFAR-10, TinyImageNet test sets. The \texttt{CNN7} architecture is used for all datasets. The certificates are obtained using IBP.}
\label{tab:main_results}
\centering
\begin{tabular}{@{}llllllll@{}}
\toprule
Dataset & $\varepsilon_{\max}$ & Method  & $\kappa$ &  Acc   &  ACR   &   ART $\uparrow$\\
\midrule  
\multirow{5}*{MNIST} &  \multirow{5}*{0.4}  &    Standard & 1.0 &  99.24 &  0.001   &   0.35    \\
&   &     CROWN-IBP & 0.0  & 98.20 &     95.13   &   96.65   \\
&   &     SABR & 0.0  & 98.75 &    87.27   &   92.83   \\
&   &     FastIBP & 0.0  &   97.61  &   94.69   &   96.13   \\
&   &     {ACERT} (ours) & 0.0  &   {98.42}   &   {95.01}   &   \textbf{96.69} \\
\midrule
\multirow{5}*{CIFAR-10} &  \multirow{5}*{$\dfrac{8.8}{255}$}  &   Standard  &   1.0     & 91.84 &   0.0002   &   0.14     \\
&   &     CROWN-IBP & 0.0 & 45.40 &   38.59   &   41.84   \\
&   &     SABR & 0.0 &  52.38 &   42.45   &   {47.14}  \\
&   &     FastIBP & 0.0 &    48.94   &   41.23   &   44.91   \\
&   &     FastIBP$^\dagger$ & 0.8 &    62.84   &   40.70   &   50.49   \\
&   &     ACERT (ours) & 0.0 &    {62.21}   &   41.80   &  {50.96}   \\
&   &     ACERT$^\dagger$ (ours) & 0.25 &    {64.63}   &   41.20   &  \textbf{51.58}   \\
\midrule
\multirow{4}*{TinyImageNet} &  \multirow{4}*{$\dfrac{1}{255}$}  &   Standard & 1.0 & 46.09  &   0.0004   &   0.13     \\
&   &     SABR & 0.0 & 28.85 &   19.99   &   23.96    \\
&   &     FastIBP & 0.0  &   25.08  &   21.06   &   22.96    \\
&   &     FastIBP$^\dagger$ & 0.85  &   30.67  &   19.90   &   24.63    \\
&   &     ACERT$^\dagger$ (ours) & 0.0   &   {32.66}   &   {21.38}   &  \textbf{26.34}     \\
\bottomrule
\end{tabular}

$^\dagger$denotes results with the highest ART score among different levels of $\kappa \in [0, 1]$ from Eq.~\eqref{eq:total_loss}.
\end{table*}

To quantitatively evaluate the certified training methods, we compare the ART (Accuracy-Robustness Tradeoff) scores in Table~\ref{tab:main_results}. We provide results for $\kappa=0$ for all methods since it is a common setup in previous works. Moreover, for the baseline FastIBP and our method, we tune the hyperparameter $\kappa \in [0, 1]$ to get the highest ART score (except for the MNIST dataset, where standard accuracies are already close to the highest possible one).  Overall, ACERT outperforms previous state-of-the-art certified training methods with and without tuning of $\kappa$ in terms of ART score. The effect of adaptive certified training is more prominent on a challenging TinyImageNet dataset, where it achieves both higher standard accuracy and average certified robust radius at the same time.

\begin{table}[tb]
\caption{Performance of certified training methods without $\varepsilon$ schedulers (denoted as $\varepsilon_t$) on CIFAR-10,  $\varepsilon_\text{max}~=~\nicefrac{8.8}{255}$. ACERT has the lowest drop in ART score, given as $\Delta$ ART, compared to default training.}
\label{tab:epsilon_scheduler}
\centering
\begin{tabular}{@{}llclllll@{}}
\toprule
Method  & $\kappa$ &  \makecell{$\varepsilon_t$} &  Acc   &  ACR   &   ART & $\Delta$~ART\\
\midrule  
 
FastIBP & 0.0 & $\times$ &  33.44 &  29.64   &   31.46 & -13.45    \\
FastIBP & 0.8 & $\times$ &  52.10 &  37.10   &   43.97 & -6.52    \\
ACERT & 0.0 & $\times$ &  52.65 &  38.95   &   45.32   & \textbf{-5.64}  \\
ACERT & 0.25 & $\times$  &  {54.41}   &   37.96   &  \textbf{45.43} & {-6.15} \\
\bottomrule
\end{tabular}
\end{table}

\subsection{Ablation Studies}
\subsubsection{The effect of the $\varepsilon$ scheduler}

Since certified training methods rely on a scheduler for certification radius, i.e. $\varepsilon_t$, during a warm-up stage of training, we investigate the effect of the scheduler on the final performance. The proposed method, ACERT, relies on a scheduler of individual maximally allowed perturbation radius, i.e. $\varepsilon_{\text{max}, t}$.

Table~\ref{tab:epsilon_scheduler} shows performances of models trained without $\varepsilon$ schedulers on CIFAR-10 dataset, where $\varepsilon_\text{max}$ is set to $\nicefrac{8.8}{255}$.
We observe that all training methods have a drop in accuracy and robustness, reflected in the $\Delta$ ART column, compared to the default training with corresponding $\varepsilon$ schedulers. Without the schedulers, ACERT has the highest ART score and the lowest drop in performance ($\Delta$ ART), compared to baseline models.
It is also notable that the performance of baseline certified training, FastIBP, improves both in terms of accuracy and robustness when the standard loss with the weight of $\kappa=0.8$ is introduced. 

Even though the proposed adaptive certified training internally assigns certification strength $\varepsilon_i$ to each input $x_i$ (providing a theoretically defined automatic scheduler), in the early stages of training the model might get overregularized whenever it processes inputs with high certified radii, limiting the final performance. Therefore, it is reasonable to cap the maximally allowed perturbation radius during a warm-up stage.

\begin{figure}[tb]
\centering
\includegraphics[width=0.93\linewidth]{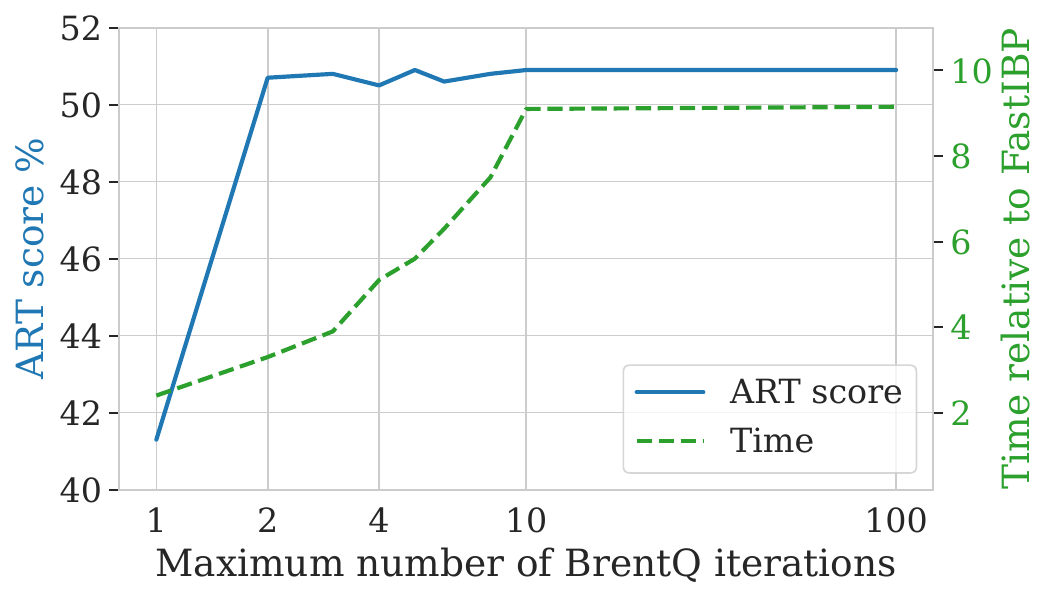}
\caption{Effect of the root finding algorithm's precision (given as a maximum number of iterations) on computation time (relative to FastIBP) and final ART score of ACERT on CIFAR-10 dataset.}
\label{fig:brentq_time_vs_iter}
\end{figure}

\subsubsection{The effect of the root finder}
\label{subsec:root_finder}
Since ACERT finds certified robust radii as roots of a scalar function $R_\text{cert}(x, \cdot)$, several additional certifications (forward passes for IBP) are required. In Fig.~\ref{fig:brentq_time_vs_iter}, we study how the precision of a root-finding algorithm, BrentQ, affects its computational efficiency and overall training performance. We train ACERT models with varying maximum numbers of BrentQ iterations, which limits the precision of the algorithm. Surprisingly, we notice that already 2 iterations of the BrentQ routine allow us to determine roots of $R_\text{cert}(x, \cdot)$ with sufficient precision to reach high ART scores. Since it runs 3 additional certifications (2 without gradients and 1 for loss computation), it is around 3.3 times slower than FastIBP's training step. Further optimization of the root finding algorithm's efficiency, such as saving and using history from previous epochs, can additionally reduce training times.

\section{Related Work}
\paragraph{Adversarial and Certified Defences}
In recent years, there has been significant research interest in developing robust training methods to enhance the resilience of machine learning models against adversarial attacks. Empirical methods~\citep{szegedy2014intriguing,madry2018towards} achieve strong robustness against various attack methods. However, adversarial training often suffers from high computational costs and can be limited in its ability to provide theoretical guarantees. To address the limitations of adversarial training, researchers have explored certified defenses that provide provable guarantees. These defenses encompass various approaches, including methods based on semidefinite relaxations~\citep{raghunathan2018certified}, linear relaxations~\citep{zhang2018efficient,xu2020automatic}, dual problems~\citep{wong2018provable,wong2018scaling,dvijotham18dual}, interval bound propagation~\citep{mirman2018differentiable,gowal2019ibp,shi2021fast} and combinations of approaches~\citep{zhang2019crownibp,balunovic2020adversarial,de2022ibp-r,muller2023sabr}. Models, trained using the certified defenses, greatly improve verified robustness, but drastically reduce standard accuracy. \citet{mueller2021certify} addresses this issue by proposing a compositional architecture of highly accurate and robust models with a selection mechanism. In our work, we focus on training a single model that achieves optimal levels of both accuracy and robustness, which can potentially be employed by compositional methods.


\paragraph{Adaptive Training} The idea of adaptive adversarial training is showing promising results for empirical defenses. \citet{Ding2020MMA, zhang2021geometryaware,cheng2022cat} propose modifications to adversarial training. \citet{xu2023exploring} motivate adaptive strategy by studying the dynamics of the decision boundary. While these methods find adaptive perturbation magnitudes heuristically (e.g. by an early stop in PGD, a fixed step of line search, or a boundary attack), our method provides theoretically sound definition {and} an efficient implementation for finding the exact values. Therefore, connecting theoretical formulations with practical execution allows further research tied to the observed phenomena.

\section{Conclusion}
We proposed a new method called ACERT that improves the accuracy-robustness tradeoffs of the existing certified training methods. It is based on the key insight that adaptive certification of $\ell_p$ balls helps to reduce overregularization and improves both standard accuracy and robustness at various perturbation levels. Moreover, we provide a theoretically sound formulation of the certified radius with an efficient GPU-accelerated exact solver for finding it. We introduce a new metric called ART (short for Accuracy-Robustness Tradeoff) that serves as a useful proxy for finding models from the accuracy-robustness Pareto front. Finally, our work establishes a new perspective on optimizing certified robust models, that encourages the development of more robust {and} accurate machine learning models suitable for deployment in real-world scenarios.

\bibliography{main}
\bibliographystyle{icml2023}

\newpage
\appendix
\onecolumn


\section{Theoretical Results}
In the following, we provide background for the certification method used, Interval Bound Propagation~\citep{gowal2019ibp}, and give proofs of the theorems from the main paper.

We assume an $L$-layer feed-forward neural network $f_\theta$ for simplicity, denoting $z^{(L)}~=~f_\theta(x)$ and $u^{(0)}~=~x$: 
\begin{align}\label{eq:feed_forward}
    z^{(\ell)} &= W^{(\ell)} u^{(\ell-1)} + b^{(\ell)},  \forall \ell \in \{1, \ldots, L\}  \\
    u^{(\ell)} &= \sigma^{(\ell)}(z^{(\ell)}), \quad \forall \ell \in \{1, \ldots, L-1\}   \label{eq:feed_forward2}
\end{align}
where $\theta=\{W^{(\ell)}, b^{(\ell)}\}_{\ell=1}^L$ is a set of affine parameters, $\sigma^{(\ell)}$ are element-wise continuous monotonic non-decreasing activation functions. We use $u$ to represent post-activation neurons and $z$ to represent pre-activation neurons. For $u^{(\ell)}$ and $z^{(\ell)}$, we define element-wise upper and lower bounds as $\underline{u}^{(\ell)} \leq u^{(\ell)} \leq \overline{u}^{(\ell)}$ and $\underline{z}^{(\ell)} \leq z^{(\ell)} \leq \overline{z}^{(\ell)}$.

\subsection{Interval Bound Propagation}\label{appendix:ibp}
Consider input bounds $\underline{u}^{(0)} \leq u^{(0)} \leq \overline{u}^{(0)}$, where $u^{(0)}=x$, $\underline{u}^{(0)}=\underline{u}^{(0)}(x, \varepsilon)$,  $\overline{u}^{(0)}=\overline{u}^{(0)}(x, \varepsilon)$. Interval bound propagation (IBP) uses simple rules to propagate bounds in a forward manner. For affine layers, it reads:
\begin{align}
    \overline{z}^{(\ell)} = W^{(\ell)} \frac{\overline{u}^{(\ell-1)} + \underline{u}^{(\ell-1)}}{2} + \abs{W^{(\ell)}} \frac{\overline{u}^{(\ell-1)} - \underline{u}^{(\ell-1)}}{2} + b^{(\ell)}, \\
    \underline{z}^{(\ell)} = W^{(\ell)} \frac{\overline{u}^{(\ell-1)} + \underline{u}^{(\ell-1)}}{2} - \abs{W^{(\ell)}} \frac{\overline{u}^{(\ell-1)} - \underline{u}^{(\ell-1)}}{2} + b^{(\ell)},
\end{align}
where $\abs{W^{(\ell)}}$ is an element-wise absolute value of ${W^{(\ell)}}$. \\
For element-wise (monotonic non-decreasing) activation functions $\sigma^{(\ell)}$, the bounds are obtained as
\begin{align}
    \overline{u}^{(\ell)} = \sigma^{(\ell)}(\overline{z}^{(\ell)}), \\
    \underline{u}^{(\ell)} =  \sigma^{(\ell)}(\underline{z}^{(\ell)}).
\end{align}

The bound propagating approach can be extended to general computational graphs, following~\citep{xu2020automatic}. We use the implementation of bound propagation from \texttt{auto\_LiRPA}\footnote{ We use the implementation of bound propagation from \url{https://github.com/Verified-Intelligence/auto_LiRPA}} library, which allows it.

\subsection{Certified Robust Radius}\label{appendix:proofs}
Below, we refresh the definition of certified robust radius from the main paper and restate the theorem~3.2 with a given proof.

\begin{def31}[Certified robust radius]\label{def:certified_robust_radius_supp}
    A bounded certified robust radius $\mathcal{E}(x,y; f_\theta; p)~\in~[0, \varepsilon_{\max}]$ of an input sample $(x, y) \sim \mathcal{D}$ with respect to the classifier $f_\theta$ is the maximum radius of an $\ell_p$ ball $\mathcal{B}(x, \varepsilon)$, not exceeding the predefined (perceptual) limit $\varepsilon_{\max}$, where the classifier is certifiably robust. Formally,
    \begin{equation}\label{eq:def_certified_radius_supp}
        \mathcal{E}_\text{cert}(x,y;f_\theta,p) = \begin{cases}
            0,& \text{if } R_\text{cert}(x, 0) \geq 0,\\
            \varepsilon_{\max}, & \text{if } R_\text{cert}(x, \varepsilon_{\max}) < 0,\\
            \sup\limits_{R_\text{cert}(x, \varepsilon) < 0} \varepsilon, & \text{otherwise}.
        \end{cases}
    \end{equation}
\end{def31}

\begin{thm32}\label{theorem:certified_radius_supp}
    Let $R_\text{cert}(x, \cdot)$ be continuous and strictly monotonically increasing {function}. Then the certified robust radius, defined in Def. 3.1, exists and is uniquely defined as 
    \begin{equation}\label{eq:certified_radius_theorem_supp}
        \mathcal{E}_\text{cert}(x,y;f_\theta,p) = \begin{cases}
            0,& \text{if } R_\text{cert}(x, 0) > 0,\\
            \varepsilon_{\max}, & \text{if } R_\text{cert}(x, \varepsilon_{\max}) < 0,\\
            \varepsilon \text{ s.t. } R_\text{cert}(x, \varepsilon) = 0, & \text{otherwise}.
        \end{cases}
    \end{equation}
\end{thm32}
\begin{proof}
   By intermediate value theorem, continuity of $R_\text{cert}(x, \cdot)$ on a bounded interval $[0, \varepsilon_\text{max}]$ leads to the existence of solution in~\eqref{eq:certified_radius_theorem_supp}. Strict monotonicity ensures the uniqueness of the obtained values of $R_\text{cert}(x, \cdot)$, hence leading to a unique solution in~\eqref{eq:certified_radius_theorem_supp}. 
   
   Now we show that the values $\mathcal{E}_\text{cert}(x,y;f_\theta,p)$ from~\eqref{eq:certified_radius_theorem_supp} are equal to the values in~\eqref{eq:def_certified_radius_supp} for given function $R_\text{cert}(x, \cdot)$. Consider $R_\text{cert}(x, 0) = 0$, then in~\eqref{eq:certified_radius_theorem_supp},~ $\mathcal{E}_\text{cert}(x,y;f_\theta,p)=0$ by the third case, and in~\eqref{eq:def_certified_radius_supp}~ by the first case. Since for  $R_\text{cert}(x, \varepsilon_\text{max}) < 0$, the definitions are the same, let us consider the certified margin function where $R_\text{cert}(x, 0) < 0$, and $R_\text{cert}(x, \varepsilon_\text{max}) \geq 0$. Due to monotonicity, the set of $\varepsilon$, where $R_\text{cert}(x, \varepsilon) < 0$, is a half-open interval $[0, \varepsilon^*)$. And due to continuity, the limit of $\sup_{[0, \varepsilon^*)} \varepsilon$ is reached at $\varepsilon^*$, s.t. $R_\text{cert}(x, \varepsilon^*) = 0$. Because of the strict monotonicity assumption, the set of $\varepsilon$ where $R_\text{cert}(x, \varepsilon) = 0$ is a single point (solution of~\eqref{eq:certified_radius_theorem_supp}), which is equal to $\varepsilon^*$, solution of~\eqref{eq:def_certified_radius_supp}.
\end{proof}

As noted in \textit{Remark} 3.3, we observe that in the trained neural networks the assumptions of continuity and strict monotonicity of IBP-certified $R_\text{cert}(x, \cdot)$ hold in practice. However, the assumption of strict monotonicity can be relaxed with a simple modification in the definition.

\begin{theorem}
\label{thm:relaxed_certified_radius}
     Let $R_\text{cert}(x, \cdot)$ be continuous and monotonically non-decreasing {function}. Then the certified robust radius, defined in Def. 3.1, exists and is uniquely defined as 
    \begin{equation}\label{eq:certified_radius_theorem2}
        \mathcal{E}_\text{cert}(x,y;f_\theta,p) = \begin{cases}
            0,& \text{if } R_\text{cert}(x, 0) > 0,\\
            \varepsilon_{\max}, & \text{if } R_\text{cert}(x, \varepsilon_{\max}) < 0,\\
            \min \varepsilon \text{ s.t. } R_\text{cert}(x, \varepsilon) = 0, & \text{otherwise}.
        \end{cases}
    \end{equation}
\end{theorem}
\begin{proof}
    Similar to the proof of Theorem~3.2, let us consider the only different case, where certified margin function $R_\text{cert}(x, 0) < 0$, and $R_\text{cert}(x, \varepsilon_\text{max}) \geq 0$. Again, due to monotonicity, the set of $\varepsilon$, where $R_\text{cert}(x, \varepsilon) < 0$, is a half-open interval $[0, \varepsilon^*)$. And due to continuity, the limit of $\sup_{[0, \varepsilon^*)} \varepsilon$ is reached at $\varepsilon^*$, s.t. $R_\text{cert}(x, \varepsilon^*) = 0$. On the other hand, since $R_\text{cert}(x, \cdot)$ is monotonous, the set of $\varepsilon$ s.t. $R_\text{cert}(x, \varepsilon) = 0$ is a closed interval $[a, b]$ (including the case of degenerate interval $[a, a]$). Since $a =  \min \varepsilon \text{ s.t. } R_\text{cert}(x, \varepsilon) = 0$, and due to monotonicity, $R_\text{cert}(x, \varepsilon) < 0$ on $[0, a)$, it follows that $a=\varepsilon^*$.
\end{proof}

Further, we show that the IBP-certified margin function $R_\text{cert}(x, \cdot)$ is continuous and monotonic for feed-forward neural networks $f_\theta$ with monotonic activation functions, such as \texttt{ReLU}, \texttt{tanh}, sigmoid, etc.

\begin{theorem}
\label{thm:monotonic_ibp}
    The IBP-certified margin function $R_\text{cert}(x, \cdot)$ is continuous and monotonically non-decreasing w.r.t. second argument $\varepsilon$ for feed-forward neural networks $f_\theta$, given as~\eqref{eq:feed_forward},\eqref{eq:feed_forward2}, and input bounds $\underline{u}^{(0)} = x - \varepsilon \mathbf{1}$ and $\overline{u}^{(0)} = x + \varepsilon \mathbf{1}$, where $\mathbf{1}$ is a vector of ones.
\end{theorem}
\begin{proof}
    Let $\varepsilon_1 > \varepsilon_2$. Then, the corresponding input bounds are
    \begin{align*}
        \overline{u}_1^{(0)} = x + \varepsilon_1 \mathbf{1}, \quad \underline{u}_1^{(0)} = x - \varepsilon_1 \mathbf{1}, \\
        \overline{u}_2^{(0)} = x + \varepsilon_2 \mathbf{1}, \quad \underline{u}_2^{(0)} = x - \varepsilon_2 \mathbf{1}.
    \end{align*}
    Let us define 
    \begin{align*}
        \Delta \overline{u}^{(\ell)} := \overline{u}_1^{(\ell)} - \overline{u}_2^{(\ell)}, \quad \Delta \underline{u}^{(\ell)} := \underline{u}_1^{(\ell)} - \underline{u}_2^{(\ell)} \\
        \Delta \overline{z}^{(\ell)} := \overline{z}_1^{(\ell)} - \overline{z}_2^{(\ell)}, \quad \Delta \underline{z}^{(\ell)} := \underline{z}_1^{(\ell)} - \underline{z}_2^{(\ell)}.
    \end{align*}
    We observe that for the IBP-defined pre-activation bounds
    \begin{align}
        \Delta \overline{z}^{(\ell)} =
 \left(W^{(\ell)}+\abs{W^{(\ell)}}\right)\frac{\Delta \overline{u}^{(\ell-1)}}{2} + \left(W^{(\ell)}-\abs{W^{(\ell)}}\right)\frac{\Delta \underline{u}^{(\ell - 1)}}{2}\geq 0, \label{eq:delta_preactivation1}  \\
        \Delta \underline{z}^{(\ell)} = 
 \left(W^{(\ell)}-\abs{W^{(\ell)}}\right)\frac{\Delta \overline{u}^{(\ell-1)}}{2} + \left(W^{(\ell)}+\abs{W^{(\ell)}}\right)\frac{\Delta \underline{u}^{(\ell - 1)}}{2}\leq 0. \label{eq:delta_preactivation2}
    \end{align}
    That is, the upper bounds of pre-activations are monotonically increasing, and the lower bounds are decreasing w.r.t. $\varepsilon$ with the preceding post-activation bounds $\overline{u}_1^{(\ell - 1)} \geq \overline{u}_2^{(\ell - 1)}$,~ $\underline{u}_1^{(\ell - 1)} \leq \underline{u}_2^{(\ell - 1)}.$ Continuity holds since the transformations are affine. Next, the monotonicity of post-activation bounds remains the same as of the pre-activation bounds, since the activations  $\sigma^{(\ell)}$ are monotonically increasing functions, and continuity follows from continuity of $\sigma^{(\ell)}$. Thus, recursively, the upper bounds of outputs $\overline{z}^{(L)}=\overline{z}^{(L)}(x, \varepsilon)$ are monotonically increasing w.r.t. $\varepsilon$, and the lower bounds $\underline{z}^{(L)}=\underline{z}^{(L)}(x, \varepsilon)$ are monotonically decreasing. Since the certified margin function is a maximum over a difference of increasing and decreasing functions, i.e. $R_\text{cert}(x,\varepsilon)~=~\max_{i \neq y}\{\overline{z}^{(L)}(x, \varepsilon)_i~-~\underline{z}^{(L)}(x, \varepsilon)_y \}$, it is monotonically increasing (or equivalently, non-decreasing) w.r.t. $\varepsilon$ and continuous.
\end{proof}

\begin{remark}
    The relaxed assumptions from Theorem~\ref{thm:relaxed_certified_radius}, and the properties of IBP-certified neural networks from Theorem~\ref{thm:monotonic_ibp} allow us to define the certified robust radius for feed-forward neural networks. Moreover, the definition in Theorem~~\ref{thm:relaxed_certified_radius} requires only a small modification in root finding algorithms to output the leftmost root, similar to \texttt{bisect\_left}. Therefore, it allows us to efficiently find the exact certified radii in practice.
\end{remark}

\begin{remark}
    The assumption of strict monotonicity of $R_\text{cert}(x, \cdot)$ holds if the activation functions~$\sigma^{(\ell)}$ are strictly increasing, e.g. sigmoid and \texttt{tanh}, and the inequalities in~\eqref{eq:delta_preactivation1} and \eqref{eq:delta_preactivation2} are strict.
\end{remark}

To summarize, the definition of certified robust radius from Theorem~3.2 of the main paper allows simplifying the presentation of algorithms while requiring strict monotonicity property of $R_\text{cert}(x, \cdot)$. However, a more general definition from Theorem~\ref{thm:relaxed_certified_radius} does not require strict monotonicity and suits a large class of neural networks.

\section{Experimental Results}\label{appendix:experimental_details}

\begin{table}[tb]
\caption{Convolutional network architecture from~\citep{shi2021fast}. Each affine layer (except the last) is followed by batch normalization and ReLU activation. "\texttt{Conv2D} $k$ $w$$\times$$h$+$s$" represents a 2D convolutional layer with $k$ filters of size $w$$\times$$h$ and stride $s$. }
\label{tab:cnn7}
\centering
\begin{tabular}{@{}c@{}}
\toprule
\texttt{CNN7} \\
\midrule
\texttt{Conv2D} 64 3$\times$3+1 \\
\texttt{Conv2D} 64 3$\times$3+1 \\
\texttt{Conv2D} 128 3$\times$3+2 \\
\texttt{Conv2D} 128 3$\times$3+1 \\
\texttt{Conv2D} 128 3$\times$3+1 \\
\texttt{Linear} 512 \\
\texttt{Linear} out\_size \\
\bottomrule
\end{tabular}
\end{table}

\paragraph{Hyperparameters} To reproduce the results of the baseline method, FastIBP~\citep{shi2021fast}, we follow the experimental setup from the authors' implementation\footnote{To reproduce the results of the baseline method, FastIBP, we follow the authors' implementation \url{https://github.com/shizhouxing/Fast-Certified-Robust-Training}}.

For the proposed method, ACERT, we use the same setup as in FastIBP, i.e. 7-layered convolutional network \texttt{CNN7} (see Table~\ref{tab:cnn7}), IBP initialization, batch sizes, total numbers of epochs, warm-up schedulers, ReLU stability and bound tightness regularizers, gradient clipping parameters, and weight decays. Find the above hyperparameters in Table~\ref{tab:hyperparams}.

We find that training of our proposed method benefits from a cyclic learning rate scheduler, while FastIBP does not. Therefore, we use \texttt{OneCycleLR} with linear annealing strategy and maximum learning rate of \texttt{2e-3} for all datasets. Please find more details of our implementation in the provided code.

\begin{table}[tb]
\caption{Experimental setups used in ACERT.}
\label{tab:hyperparams}
\centering
\begin{tabular}{@{}lccc@{}}
\toprule
& MNIST & CIFAR-10 & TinyImageNet \\
\midrule
Model & \texttt{CNN7} & \texttt{CNN7} & \texttt{CNN7} \\
Initialization & IBP & IBP & IBP \\
$\varepsilon_{\text{max}}$ & 0.4 & $\nicefrac{8.8}{255}$ & $\nicefrac{1}{255}$ \\
Batch size & 128 & 128 & 256 \\
Num epochs & 100 & 200 & 200 \\
Warm-up schedules & 1-20 & 1-80 & 1-80 \\
Reg. coefficient & 0.5 & 0.5 & 0.2 \\
$L_1$ weight decay & 0.0 & 0.0 & \texttt{1e-4} \\
Grad. clipping norm & 10.0 & 10.0 & 10.0 \\
LR scheduler & \texttt{OneCycle} & \texttt{OneCycle} & \texttt{OneCycle} \\
\texttt{max\_lr} & \texttt{2e-3} & \texttt{2e-3} & \texttt{2e-3} \\
\bottomrule
\end{tabular}
\end{table}

\paragraph{Error bars} We provide averaged results of our proposed method, ACERT, over multiple runs with different seeds on the CIFAR-10 dataset with $\varepsilon_\text{max}~=~\nicefrac{8.8}{255}$ in Table~\ref{tab:error_bars}. The conclusions remain the same as in the main paper, where the (best) single runs are presented for all methods.

\begin{table}[tb]
\caption{Means and standard deviations over 7 runs on CIFAR-10 with  $\varepsilon_\text{max}~=~\nicefrac{8.8}{255}$.}
\label{tab:error_bars}
\centering
\begin{tabular}{@{}llllll@{}}
\toprule
Method  & $\kappa$ &   Acc   &  ACR   &   ART\\
\midrule  
ACERT & 0.0 &  61.64$\pm$0.34 &  41.79$\pm$0.19   &   50.74$\pm$0.13\\
ACERT & 0.25  &  64.35$\pm$0.16   &   40.48$\pm$0.32   &  \textbf{51.02}$\pm$0.24  \\
\bottomrule
\end{tabular}
\end{table}

\paragraph{Running times}
The experiments were all run on a single Tesla V-100 GPU. The implementations of both bound propagation and root finding algorithms are within PyTorch~\citep{paszke2019pytorch} framework, allowing GPU acceleration.

We report the running times of the proposed method and its variations in comparison to the baseline method, FastIBP, in Table~\ref{tab:running_times}. The total runtime includes both training and validation epochs. For FastIBP, the validation step also computes certified radii. We observe that only two iterations of the root finding algorithm, BrentQ~\citep{brent1973algorithmsfor}, are sufficient to achieve a high ART score. Therefore, making the runtime of our method comparable to FastIBP.

\begin{table}[tb]
\caption{Comparison of variations of the proposed method, ACERT, with different running times against the baseline on CIFAR-10 with $\varepsilon_\text{max}~=~\nicefrac{8.8}{255}$. All models are trained with the percentage of standard loss $\kappa=0$.}
\label{tab:running_times}
\centering
\begin{tabular}{@{}lllllccc@{}}
\toprule
Method  & \makecell{Max \\BrentQ\\ iters} &  Acc   &  ACR   &   ART & \makecell{Training \\ step, sec} & \makecell{Total \\ runtime, h}\\
\midrule  
FastIBP & 0 &  48.94 &  41.23   &   44.91 & 0.02 & 6.1  \\
ACERT-fast & 0 &  54.35 &  42.37   &   47.92 &  0.02 & 6.1  \\
ACERT & 2 &  61.82 &  41.66   &   50.70 &  0.07 & 7.2  \\
ACERT & 100 &  61.64 &  41.79   &   \textbf{50.74} & 0.22 & 10.0  \\
\bottomrule
\end{tabular}
\end{table}

An extreme variation of ACERT with 0 root finding iterations, ACERT-fast (given in the second row of Table~\ref{tab:running_times}), employs the following strategy. For misclassified samples, it sets the certified radius to 0, and for correctly classified ones -- to $\varepsilon_{\text{max}, t}$. So, in contrast to FastIBP, which sets $\varepsilon_{\text{max}, t}$ for all samples, this variation adapts to the classification result. Since it does not compute the exact certified radii, its runtime is the same as for FastIBP. Interestingly, this zero-iteration adaptation already improves the accuracy and average certified radius over the baseline but does not reach the level of ACERT.

\begin{figure}[tb]
\centering
\begin{tabular}{cc}
\includegraphics[width=0.44\linewidth]{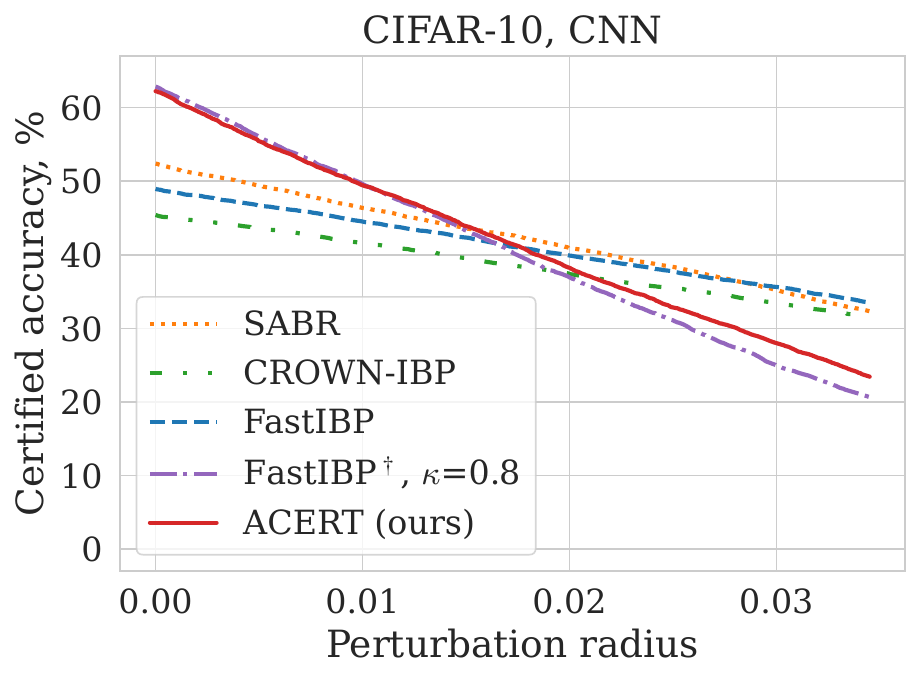} &
\includegraphics[width=0.44\linewidth]{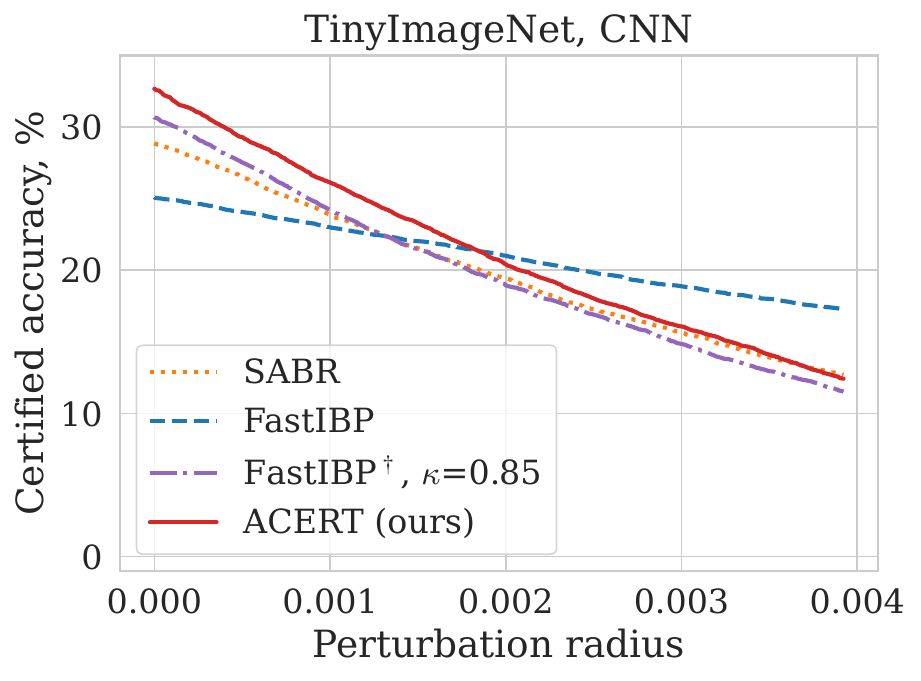} 
\end{tabular}
\caption{Certified robust accuracies on different levels of $\ell_\infty$ perturbation radii. The area under a curve is equal to the (unnormalized) average certified radius (ACR) of the method. By default, models with a percentage of standard loss $\kappa=0$ are demonstrated.}
\label{fig:certified_robustness_curves}
\end{figure}

\begin{table}[tb]
\caption{Comparison of certified robust accuracies (Cert. Acc.) at commonly used fixed levels of $\ell_\infty$ perturbation radii $\varepsilon_\text{test}$. The certificates are obtained using IBP.}
\label{tab:robust_accuracies_results}
\centering
\begin{tabular}{@{}lllllll@{}}
\toprule
Dataset & $\varepsilon_{\text{test}}$ & Method  & $\kappa$ &  Acc   &  \makecell{Cert. \\ Acc.}  \\
\midrule  
\multirow{2}*{MNIST} &  \multirow{2}*{0.3}  &    FastIBP & 0.0 &  97.61 &  93.13     \\
&   &     {ACERT} (ours) & 0.0  &   {98.42}   &   {93.20}    \\
\midrule
\multirow{4}*{CIFAR-10} &  \multirow{4}*{$\dfrac{8}{255}$}  &   FastIBP  &   0.0     & {48.94}   &   {34.97}      \\
&   &     FastIBP$^\dagger$ & 0.8 &    {62.83}   &   23.71      \\
&   &     ACERT (ours) & 0.0 &    {62.21}   &   26.61      \\
&   &     ACERT$^\dagger$ (ours) & 0.25 &    {64.63}   &   24.41      \\
\midrule
\multirow{3}*{TinyImageNet} &  \multirow{3}*{$\dfrac{1}{255}$}  &   FastIBP & 0.0 & 25.08  &   {17.29}      \\
&   &     FastIBP$^\dagger$ & 0.85   &   {30.67}   &   {11.51}       \\
&   &     ACERT$^\dagger$ (ours) & 0.0   &   {32.66}   &   {12.42}       \\
\bottomrule
\end{tabular}

$^\dagger$denotes results with the highest ART score among different levels of $\kappa \in [0, 1]$.
\end{table}

\section{Limitations and Discussion}
To get a closer look at individual performances, we evaluate certifiably trained models by calculating certified robust accuracies at different levels of $\ell_\infty$ perturbation radii, given in Figure~\ref{fig:certified_robustness_curves} and Table~\ref{tab:robust_accuracies_results}. We observe in Fig.~\ref{fig:certified_robustness_curves} that our method has higher robustness at smaller perturbation intensities and finishes with lower certified accuracy at commonly defined large $\varepsilon_\text{test}$ (Table~\ref{tab:robust_accuracies_results}). This limitation of our method is related to accuracy-robustness tradeoffs, discussed in Section~\ref{subsec:main_results} and demonstrated in Fig.~\ref{fig:accuracy_robustness_curves}. To illustrate it, we tune the baseline method, FastIBP, to have higher standard accuracy (at the level of ACERT), and see that its certified accuracy at large $\varepsilon_\text{test}$ drops below ACERT's.

In the field of certified training, researchers have mostly focused on achieving high certified accuracy at specific large epsilon values, even if it means sacrificing standard accuracy and certified accuracy at smaller epsilon ranges. However, this often leads to models that are not practical in real-world scenarios. Moreover, there is a concern that the advancements in the field could be attributed to selectively improving performance points on the same accuracy-robustness curve. Therefore, we argue that the evaluation should focus on optimizing the entire tradeoff. To achieve this, we propose to use the ($\kappa$-tuned) ART score~\eqref{eq:ART} as a useful proxy metric, indicating a method's Pareto front.

\section{Broader Impact}
Since the proposed method bridges the gap between practical usability and safety of the trained models, it {can} help make real-world AI applications more resilient to adversaries, providing robustness guarantees. While ACERT achieves state-of-the-art accuracy-robustness tradeoffs, it does not necessarily indicate sufficient robustness for safety-critical applications but could give a false sense of security.

\end{document}